\documentclass[11pt,letterpaper]{article}

\usepackage[T1]{fontenc}
\usepackage[utf8]{inputenc}
\usepackage{lmodern}
\usepackage[margin=1in]{geometry}
\usepackage{microtype}
\usepackage{amsmath,amssymb,amsthm,mathtools}
\usepackage{booktabs,tabularx,array}
\usepackage{enumitem}
\usepackage{xcolor}
\usepackage[ruled,vlined,linesnumbered]{algorithm2e}
\usepackage{verbatim}
\usepackage[section]{placeins}
\usepackage[authoryear,round]{natbib}
\usepackage{hyperref}
\usepackage{aliascnt}
\usepackage[nameinlink,capitalise,noabbrev]{cleveref}
\usepackage{fancyhdr}

\definecolor{linkblue}{RGB}{25,65,110}
\definecolor{notebg}{RGB}{247,248,250}
\setlist{leftmargin=1.6em,itemsep=2pt,topsep=4pt}
\allowdisplaybreaks
\numberwithin{equation}{section}

\newtheorem{theorem}{Theorem}[section]
\newaliascnt{proposition}{theorem}
\newtheorem{proposition}[proposition]{Proposition}
\aliascntresetthe{proposition}
\newaliascnt{lemma}{theorem}
\newtheorem{lemma}[lemma]{Lemma}
\aliascntresetthe{lemma}
\newaliascnt{corollary}{theorem}

\aliascntresetthe{corollary}
\theoremstyle{definition}
\newaliascnt{definition}{theorem}
\newtheorem{definition}[definition]{Definition}
\aliascntresetthe{definition}
\newaliascnt{remark}{theorem}

\aliascntresetthe{remark}
\crefname{lemma}{Lemma}{Lemmas}
\crefname{proposition}{Proposition}{Propositions}
\newcommand{\E}{\mathbb E}
\newcommand{\Pp}{\mathbb P}
\newcommand{\ind}{\mathbf 1}
\newcommand{\eps}{\varepsilon}
\newcommand{\wh}{\widehat}
\newcommand{\Reg}{\operatorname{Reg}}

\newcommand{\cZ}{\mathcal Z}

\newcommand{\cC}{\mathcal C}
\DeclareMathOperator*{\argmax}{arg\,max}
\SetKwInOut{Input}{Input}
\SetKwInOut{Output}{Output}
\SetKw{Return}{return}
\SetKwFunction{DP}{BackwardInduction}
\DontPrintSemicolon

\title{Optimal No-Regret Learning for Repeated Prophet Inequality}
\author{Kun Wang\\Purdue University}
\date{\today}

\begin{document}
\maketitle
\thispagestyle{plain}
\vspace{-1.5em}

\begin{abstract}
We study repeated prophet inequalities under prefix feedback. In each of
$T$ rounds, a learner encounters fresh values drawn independently from
$n$  boxes with unknown $[0,1]$-supported distributions in a fixed order and must
irrevocably accept one, observing only the prefix up to its stopping box.
Regret is measured against the optimal stopping policy that knows the distributions. We give an efficient algorithm achieving
$\widetilde O(\sqrt{T})$ expected regret, matching the lower bound up to logarithmic factors.
Our algorithm explores directly through near-optimal policies, combining empirical backward induction with box-specific reach bonuses. A relative-drop aggregation rule then exploits the nesting structure
of observed prefixes to preserve exploration, thereby removing the polynomial dependence on the box number $n$. This resolves an open question posed by \cite{liu2025improved}.
\end{abstract}

\section{Introduction}
\label{sec:introduction}

Prophet inequalities study a fundamental question in sequential
decision-making: how should one choose among opportunities that arrive
one at a time, when accepting an opportunity prevents all future choices
and rejecting it is irreversible? In the classical problem, the rewards
are independent nonnegative random variables with known distributions. A stopping
rule observes their realizations sequentially and selects at most one,
whereas a prophet sees all realizations before choosing. The classical
prophet inequality guarantees one half of the prophet's expected reward,
and this factor is tight in general
\citep{krengel1977semiamarts,samuelcahn1984comparison}.
This theory quantifies the cost of not knowing future realizations, but
leaves a separate learning question: what happens when the distributions
are also unknown?

We study this question in the repeated prophet inequality problem with
\emph{prefix feedback}. There are $n$ boxes with unknown distributions
supported on $[0,1]$. In each of $T$ rounds, fresh values are drawn
independently across boxes and rounds, and the learner inspects the boxes
in a fixed order. Upon observing a value, it either accepts that value
and ends the round or rejects it and continues. Consequently, it observes
only the prefix ending at its stopping box. The benchmark is the optimal
stopping policy that knows the distributions but faces the same
sequential decisions, not the prophet that knows all realized values.
Our goal is to minimize cumulative expected regret against this benchmark.
The central difficulty is that a stopping decision determines both the
reward earned now and the information available for future rounds.

An important starting point is the 2024 result of
\citet{jin2024sample}: an $\varepsilon$-optimal stopping policy can be
learned from $\widetilde O(\varepsilon^{-2})$ independent samples of the
entire value vector, with no dependence on $n$. This already implies an
$n$-independent $\widetilde O(T^{2/3})$ regret bound under prefix feedback.
Indeed, the learner can traverse all boxes for $m$ rounds, learn a policy
from the resulting complete samples, and use that policy thereafter.
The exploration rounds cost at most $m$ regret, while the learned policy
has suboptimality $\widetilde O(m^{-1/2})$. Thus this
\emph{explore-then-commit} reduction gives
\[
  \Reg_T
  \le m+\widetilde O\!\left(\frac{T-m}{\sqrt m}\right)
  =\widetilde O(T^{2/3})
  \qquad\text{for }m\asymp T^{2/3}.
\]
The limitation is the cost of obtaining the samples: a full traversal
can sacrifice a constant amount of reward, even when only a small
improvement in policy accuracy is needed.

A complementary line of work achieves the optimal $\sqrt T$ dependence,
but with polynomial dependence on the number of boxes. Under prefix
feedback, \citet{gatmiry2024bandit} obtain $\widetilde O(n^3\sqrt T)$ regret for the first time. \citet{agarwal2024semibandit} obtain
$\widetilde O(n\sqrt T)$ regret for any monotone stochastic optimization problem (including prophet inequality). Subsequently, \citet{liu2025improved} improve this
to $\widetilde O(\sqrt{nT})$. Meanwhile, an $\Omega(\sqrt T)$ lower bound
holds even with two boxes \citep{gatmiry2024bandit}. The two available
guarantees therefore offer different advantages: the consequence of
\citet{jin2024sample} avoids dependence on $n$ but has a $T^{2/3}$ horizon
dependence, whereas the bound of \citet{liu2025improved} has the desired
$\sqrt T$ dependence but pays a factor of $\sqrt n$. Recognizing this
contrast, \citet{liu2025improved} explicitly ask whether both advantages
can be obtained simultaneously:
\begin{quote}
\emph{Can repeated prophet inequality under prefix feedback be learned
with $\widetilde O(\sqrt T)$ regret, with only logarithmic dependence on
$n$?}
\end{quote}

We resolve this question affirmatively. We give a polynomial-time
algorithm that, for arbitrary independent $[0,1]$-supported distributions,
achieves $\Reg_T=O\!\left(\sqrt T\,[1+\log(nT)]^2\right)$. This matches the $\Omega(\sqrt T)$ lower bound up to logarithmic factors.
The improvement is simultaneous: we improve the $n$-independent
$\widetilde O(T^{2/3})$ guarantee to the optimal horizon exponent, while
removing the polynomial dependence on $n$ from the previous
$\widetilde O(\sqrt{nT})$ guarantee. 

\paragraph{The challenge}
 Removing the polynomial dependence on $n$ requires more than
sharpening the analysis of the existing
$\widetilde O(\sqrt{nT})$ algorithm \citep{liu2025improved}: its coordinatewise optimism
can itself cause excessive exploration. Consider the deterministic
instance
\[
    X_i=\frac12-\frac{i}{2n},
    \qquad i\in[n].
\]
The values decrease with the box index, so the optimal policy
immediately accepts the first value. Continuing to the last box,
by contrast, sacrifices a constant amount of reward.

Why would an optimistic learner keep searching on this instance?
Even after repeatedly observing the same values, its optimistic
model leaves open the possibility of a rare, much larger reward
at each box. With many boxes remaining, these individually small
possibilities accumulate, making continued exploration appear
more attractive than accepting the current value. The learner
therefore keeps searching for rewards that never materialize.
In fact, on this instance, the algorithm traverses all boxes
throughout its first $n$ rounds and incurs $\Omega(n)$ regret.
Taking $T=n$, this is substantially larger than the
$\widetilde O(\sqrt n)$ guarantee we seek. The example highlights
the central challenge: exploration must be controlled by the
total reward loss of the policy being played, rather than
driven solely by optimism at individual boxes.

\paragraph{Our approach}
We instead explore through near-optimal policies. At each accuracy
scale $\varepsilon$, we maintain a behavior policy with
$O(\varepsilon)$ suboptimality whose reach probabilities cover,
up to logarithmic factors, those of all sufficiently near-optimal
threshold policies and their initial mixtures.
Empirical backward induction constructs box-specific explorers,
and mixing them with an empirical reward-maximizing baseline
controls their reward loss. We then exploit the nesting of prefix
observations: exploring a later box also reveals every earlier box.
Our relative-drop aggregation combines these explorers with only
a logarithmic loss in coverage, rather than a factor of $n$.
A localized concentration analysis supports each refinement using
$\widetilde O(\varepsilon^{-2})$ rounds at
$O(\varepsilon)$ regret per round. Summing over geometrically
decreasing accuracy scales yields $\widetilde O(\sqrt T)$ regret.

\section{Related Work}
\label{sec:related}

\paragraph{Classical prophet inequality}
Classical prophet inequalities compare the expected reward of an online
stopping rule with the expected maximum of independent nonnegative
random variables. The foundational results of
\citet{krengel1977semiamarts,samuelcahn1984comparison} establish the
sharp $1/2$ guarantee. Subsequent work extends results to different variants, such as limited-information and
sample-based settings and other setting\citep{azar2014prophet,rubinstein2020optimal,jiang2025tight,kleinberg2012matroid,kleinberg2019matroid,feldman2021online,gravin2019prophet,ezra2022prophet}.
This line of work mainly assume known distribution.

\paragraph{Repeated prophet inequality}
For repeated prophet inequality, it assumes the distribution of each box is unknown and try to learn the distribution through repeated interactions. There are two standard settings: PAC setting and regret-minimization setting \citep{guo2019settling,guo2021generalizing,gatmiry2024bandit,agarwal2024semibandit,jin2024sample,liu2025improved}. In our work, we focus on regret-minimization setting. \citet{gatmiry2024bandit} obtain
$\widetilde O(n^3\sqrt T)$ regret under stronger bandit feedback. Under prefix feedback, \citet{agarwal2024semibandit}
give a $\widetilde O(n\sqrt T)$ regret bound. Then \citet{liu2025improved} improves this to $\tilde{O}(\sqrt{nT})$ based on similar algorithmic framework. Inspired by \cite{jin2024sample}, who gives an upper bound independent of box number $n$ in the PAC setting,  \citet{liu2025improved} ask whether the polynomial dependence on $n$ can be removed as well in the regret-minimization setting. Our $O(\sqrt T\,[1+\log(nT)]^2)$ bound addresses this question and matches the $\Omega(\sqrt T)$ lower bound
\citep{gatmiry2024bandit} up to logarithmic factors.

\section{Problem Setting and Main Guarantee}
\label{sec:model}

There are $n$ boxes with unknown distributions $D_1,\ldots,D_n$
supported on $[0,1]$. In each round $t$, a fresh vector $(X_{t,1},\ldots,X_{t,n}) \sim D:=\bigotimes_{i=1}^n D_i$
is drawn independently of all previous rounds. The learner inspects the
boxes in the fixed order $1,\ldots,n$. After observing a value, it either
accepts that value and ends the round or irrevocably rejects it and
continues. If it stops at box $S_t\in\{1,\ldots,n\}$, it receives
$X_{t,S_t}$ and observes only the prefix
$X_{t,1},\ldots,X_{t,S_t}$. We refer to this observation model as
\emph{prefix feedback}.

For a policy $\pi$, let $S_\pi$ denote its stopping location. Define $V(\pi)=\E_D[X_{S_\pi}], q_i^\pi=\Pp_D(S_\pi\ge i),V^\star=\sup_\pi V(\pi), \Delta(\pi)=V^\star-V(\pi)$,
where the expectations and probabilities include the policy's internal
randomness. Therefore, $q_i^\pi$ is the probability that $\pi$ reaches box $i$,
and $\Delta(\pi)$ is its suboptimality gap. We set $q_{n+1}^\pi=0$.
The \emph{full-traversal policy}, denoted by $F$, always stops at the last
box and therefore satisfies $q_i^F=1$ for every $i\in\{1,\ldots,n\}$ and
$\Delta(F)\le1$.

A \emph{threshold policy} stops at the first box $i<n$ whose
value satisfies $X_i\ge\tau_i$, where $\tau_i$ is the threshold of box $i$. For known distributions, \emph{backward induction} works from the
last box backward, setting the threshold at each box to the
optimal expected reward obtainable by continuing to the remaining
boxes. This yields an optimal threshold policy $\pi^\star$ satisfying
$V(\pi^\star)=V^\star$. A \emph{mixture} randomizes over policies before observing
any values in the round and follows the selected policy throughout
that round. For a mixture
$\mu=\sum_{k=1}^K w_k\pi_k$, where $w_k\ge0$ and $\sum_k w_k=1$, $V(\mu)=\sum_{k=1}^K w_k V(\pi_k)$, $q_i^\mu=\sum_{k=1}^K w_k q_i^{\pi_k}$.

Throughout the paper, we restrict attention to all finite initial mixtures of threshold policies. Denote this policy class as $\Pi$.  Let $B_t$ denote the policy selected based on the history before round
$t$. Define $ R_T=\sum_{t=1}^T\Delta(B_t)$ and expected cumulative regret as $\Reg_T=\E[\mathcal R_T]=\sum_{t=1}^T\bigl(V^\star-\E[X_{t,S_t}]\bigr)$.

\begin{theorem}[Main guarantee]
\label{thm:main}
\cref{alg:full-horizon} satisfies $\Reg_T=O\!\left(\sqrt{T}\,[1+\log(nT)]^2\right)$.
\end{theorem}

\section{The Algorithm}
\label{sec:algorithm}

In this section, we introduce the idea of  our algorithm. Our algorithm is a phase-based algorithm. At each phase, we first run a given policy $B$ for some times to collect samples from each box, then construct a new policy $B'$ as the input of next phase. The accuracy guarantee of new policy $B'$ is improved compared to input policy $B$.

First, we will introduce some notations for empirical distributions. For each box $i\in[n]$, let
$\mathcal S_i=(Y_{i,1},\ldots,Y_{i,m_i})$ be its retained sample
array, where $m_i\ge1$, and write
$\mathcal S=(\mathcal S_1,\ldots,\mathcal S_n)$.
The empirical distribution $\widehat D_i$ assigns mass $1/m_i$
to each sample and define $\widehat D=\bigotimes_{i=1}^n\widehat D_i$. Thus for any function $f$, $\mathbb E_{\widehat D_i}[f(X_i)]=\frac{1}{m_i}\sum_{s=1}^{m_i}f(Y_{i,s})$. Given a threshold policy $\pi$, let its threshold be $(\tau_1,\ldots,\tau_n)$. We define its stopping
function by $a_i^\pi(x)=\mathbf 1\{x\ge\tau_i\},\forall i<n$, and $a_n^\pi(x)=1$.
Then define empirical conditional stopping probability for policy $\pi$ as $\widehat t_i^\pi=\E_{\wh D_i}[a_i^\pi(X_i)]=\frac{1}{m_i}\sum_{s=1}^{m_i}
      \mathbf{1}\{Y_{i,s}\ge\tau_i\}$
and reach probability $\widehat q_1^\pi=1$,
    $\widehat q_{i+1}^\pi=\widehat q_i^\pi(1-\widehat t_i^\pi)\; \forall i\in [1,n)$. Its empirical reward is defined as \[\widehat V(\pi)=\mathbb E_{\widehat D}[X_{S_\pi}]=
\sum_{i=1}^n
\widehat q_i^\pi
\left(
\frac1{m_i}\sum_{s=1}^{m_i}
Y_{i,s}\mathbb{I}\{Y_{i,s}\ge\tau_i\}
\right),
\] where $S_{\pi}$ denote the stopping box of policy $\pi$. Replacing $\wh D_i$ by $D_i$ gives the corresponding true quantities, using the same realized thresholds.

At the beginning of each phase, when the algorithm collect samples, it does not use all $N_i$ samples it collect through the policy. Instead, for each box $i$, it only keeps $m_i=2^{\lfloor\log_2N_i\rfloor}$ samples to construct empirical distribution $\wh D_i$. This dyadic rounding retains
at least half of the observations at each box and simplifies
the concentration analysis by reducing the number of possible
sample-count profiles over which a union bound is taken.

After finishing collecting the samples, we construct two type policies: baseline policy and bonus policy. The baseline policy maximize the reward from empirical distribution $\wh V(\pi)$. This can be implemented by backward induction with empirical reward payoff function. The bonus policy for box $j$ try to maximize $\wh V(\pi)+z\wh q_j^\pi$. This can be implemented by modifying the payoff of each sample $Y_{i,s}$ to $Y_{i,s}+z\mathbb{I}\{i\geq j\}$. Then calculate threshold based on backward induction and run threshold policy using new payoff to decide when to stop. Next, we introduce two key technical ingredients and explain how to aggregate them into a single policy $B'$.

\begin{itemize}
    \item \textbf{Baseline mixing} A bonus policy $\pi_{j,z}$ may explore more but earn a lower empirical reward.
To control this loss, we mix each bonus policy $\pi_{j,z}$ with the baseline $\pi_0$. The mixture $\sigma_{j,z}$ selects $\pi_{j,z}$ with probability
$w_{j,z}$ and $\pi_0$ otherwise, before observing any values, and follows
the selected policy throughout the round. The score $s_{j,z}$ is the
bonus component's contribution to its empirical reach probability. Let $d_{j,z}= \wh{V}(\pi_0)-\wh V(\pi_{j,z})$. By linearity, every candidate mixture satisfies 
\begin{equation}
    \label{eq:empirical-mixture-value}
    \begin{aligned}
        \wh V(\sigma_{j,z})
    =\wh{V}(\pi_0)-\frac{a\,d_{j,z}}{a+d_{j,z}}
    \ge \wh{V}(\pi_0)-a
    \end{aligned}
\end{equation}
    Thus, mixing limits the empirical reward loss to at most $a$.
\item \textbf{Relative-drop aggregation} The final step combines the box-specific explorers into a single
behavior policy. To exploit the fact that reaching a later box also
reveals every earlier box, form a nonincreasing envelope of the
exploration scores. The full-traversal
component guarantees that every box is reached with probability at
least $\eps$, while relative-drop weighting exploits the overlap
among the prefixes observed by different explorers.
\end{itemize}

Based on the above construction in one phase algorithm, the following invariant propagation property is satisfied, which is the key lemma to prove the main theorem. We will prove it in \cref{sec:closure} after introducing all the necessary tools.
\begin{proposition}[Invariant propagation]
\label{prop:closure}
Given $G=64(1+\log(nT/\delta))$. Consider a phase $h$ with input policy $B$ and parameter $\eps$, suppose $\Delta(B)\le16\eps$ and $q_i^B \geq \max\{\frac{q_i^\pi}{G},2\eps\}\text{ for every policy $\pi$ with $\Delta(\pi)\le8\eps$}$ and $\forall i \in [n]$, then with probability $1-\frac{\delta}{2(h+1)^2}$, the phase output $B'$ satisfies  $\Delta(B')\le8\eps$ and $q_i^{B'}\ge\max\{\frac{q_i^{\pi}}{G},\eps\}\;\text{ for every policy $\pi$ with $\Delta(\pi)\le4\eps$}$ and $\forall i\in[n]$.
\end{proposition}

For the purpose of reuse, we define the input condition as follows:

\begin{definition}[Condition 1]
\label{input_condition}
Given $G=64(1+\log(nT/\delta))$. For a phase $h$ with input policy $B$ and parameter $\eps$, $\Delta(B)\le16\eps$ and $q_i^B \geq \max\{\frac{q_i^\pi}{G},2\eps\}\text{ for every policy $\pi$ with $\Delta(\pi)\le8\eps$}$ and $\forall i \in [n].$
\end{definition}

For full horizon algorithm, the algorithm proceeds in phases at geometrically decreasing accuracy
scales. At each phase, the current behavior policy $B$ is held fixed for
$M_\eps$ rounds to collect the observations needed to construct the
next behavior policy. After a successful phase, the algorithm replaces
$B$ with the phase output $B'$ and halves $\eps$. Thus, the schedule
uses smaller batches for coarse updates early on and larger batches
for finer updates later. In this way, the total regret could be controlled. See \cref{alg:phase} and \cref{alg:full-horizon} for details.

\begin{algorithm}[htbp]
\small
\caption{One phase at scale $\eps$}
\label{alg:phase}
\Input{Frozen behavior $B$; Error parameter $\eps$}
\Output{Next behavior $B'$ or FAIL}
 Define $A=1+\log(nT/\delta)$, $M_\eps=\left\lceil C_M A^4/\eps^2\right\rceil$,
  $a=6\eps$, $\cZ_\eps=\{\eps,2\eps,4\eps,\ldots,1\}$\\
Play $B$ for $M_\eps$ rounds; record all observed prefixes and counts $N_i$ for box $i$\;
\If{$\min_{i\in[n]}N_i=0$}{
  \Return{$\mathrm{FAIL}$}\;
}
Set $m_i=2^{\lfloor\log_2N_i\rfloor}$; Use the first $m_i$ observations to construct $\wh D_i$\;
$\pi_0\gets \argmax_\pi\wh V(\pi)$ by backward induction; $v_0\gets\wh V(\pi_0)$; $\sigma_1\gets\pi_0$\;
\For{$j=2,\ldots,n$}{
 \ForEach{$z\in\cZ_\eps$}{
  $ \pi_{j,z}\gets\argmax_\pi
    \{\wh V(\pi)+z\wh q_j^\pi\}$ by backward induction\;
  $d_{j,z}\gets v_0-\wh V(\pi_{j,z})$; $w_{j,z}\gets a/(a+d_{j,z})$\;
  $s_{j,z}\gets w_{j,z}\wh q_j^{\pi_{j,z}}$\;
 }
 Choose any $z_j$ maximizing $s_{j,z}$.\\
 $s_j\gets s_{j,z_j}$; $\sigma_j\gets(1-w_{j,z_j})\pi_0+w_{j,z_j}\pi_{j,z_j}$\;
}
$R_{n+1}\gets\eps$\;
\For{$i=n,n-1,\ldots,2$}{$R_i\gets\max\{s_i,R_{i+1}\}$\;}
$R_1\gets1$; $\alpha_i\gets(R_i-R_{i+1})/R_i$; $Z\gets\sum_i\alpha_i$\;
\Return{$\eps F+(1-\eps)Z^{-1}\sum_i\alpha_i\sigma_i$}\;
\end{algorithm}







\begin{algorithm}[htbp]
\caption{Full-horizon schedule}
\label{alg:full-horizon}
\Input{Horizon $T$}
$B\gets F$, $\eps\gets1/8$, $r\gets T$; $M_\eps\gets\left\lceil C_M A^4/\eps^2\right\rceil$\;
\While{$M_\eps\le r$}{
$B'\gets$ Algorithm~\ref{alg:phase}$(B,\eps)$\;
$r\gets r-M_\eps$\;
\lIf{$B'=\mathrm{FAIL}$}{\textbf{break}}
$B\gets B'$; $\eps\gets\eps/2$\;
}
Play $B$ for the remaining $r$ rounds.\;
\end{algorithm}



\section{Three Deterministic Lemmas}
\label{sec:deterministic}

In this section, we prove three deterministic lemmas that is useful to prove our main theorem. The following lemma shows empirical near-optimality of the constructed policy implies true near-optimality at each phase.

\begin{lemma}[Empirical optimal implies true optimal]
\label{lem:localization}
Let $\cC$ be closed under mixing and contain an optimal policy $\pi^\star$. Fix $\eps>0$. Assume  $|\wh V(\rho)-V(\rho)|\le \frac{\eps}{2}$
holds simultaneously for every $\rho\in\cC$ with
$\Delta(\rho)\le 8\eps$. Then every policy $\mu\in\cC$ satisfying
$\wh V(\mu)\ge \wh V(\pi^\star)-6\eps$ has $\Delta(\mu)\le 7\eps$.
\end{lemma}

\begin{proof}
Fix such a policy $\mu$ and suppose, for a contradiction, that
$\Delta(\mu)>7\eps$. Set $\lambda=\frac{7\eps}{\Delta(\mu)}\in(0,1), \nu=(1-\lambda)\pi^\star+\lambda\mu$.
By closure under mixing, $\nu\in\cC$. Linearity gives $\Delta(\nu)=\lambda\Delta(\mu)=7\eps$. Thus the accuracy assumption $|\wh V(\rho)-V(\rho)|\le \frac{\eps}{2}$ applies to both policy $\nu$ and $\pi^\star$.
Moreover, empirical linearity gives $\wh V(\pi^\star)-\wh V(\nu)
    =\lambda\bigl(\wh V(\pi^\star)-\wh V(\mu)\bigr)
    \le 6\lambda\eps$.
Combining these bounds yields $7\eps
    =V(\pi^\star)-V(\nu)
    \le \wh V(\pi^\star)-\wh V(\nu)+\eps
    \le 6\lambda\eps+\eps
    <7\eps$, where the strict inequality follows from $\lambda<1$. This contradiction proves the claim.
\end{proof}

Next two lemmas show the exploration property of the constructed policy.

\begin{lemma}[Baseline mixing preserves exploration]
\label{lem:bonus}
Let $a=6\eps$ and use the dyadic bonus grid
$\cZ_\eps$. If a policy $\pi^{\rm}$ satisfies $\wh V(\pi^{\rm })\ge v_0-a$ and $\wh q_j^{\pi^{\rm}}\ge4a$,
then $\max_{z\in\cZ_\eps}s_{j,z}
  \ge \frac14\,\wh q_j^{\pi^{\rm}}$.
\end{lemma}
\begin{proof}
Set $r=\wh q_j^{\pi^{\rm}}$. By assumption,
$4a\le r\le1$. Since $a=6\eps$, we have $\eps\le\frac{2a}{r}\le\frac12$.
Choose the smallest grid point $z\in\cZ_\eps$ satisfying
$z\ge2a/r$. Such a point exists, and the dyadic spacing gives $\frac{2a}{r}\le z\le\frac{4a}{r}$.

Write $d=v_0-\wh V(\pi_{j,z})\ge0, q=\wh q_j^{\pi_{j,z}}$.
Optimality of $\pi_{j,z}$ implies  $v_0-d+zq \ge \wh V(\pi^{\rm})+z\wh q_j^{\pi^{\rm}}
  \ge v_0-a+zr$. Because $zr\ge2a$, it follows that
$zq\ge d+zr-a\ge d+a$.
Therefore, the exploration score after baseline mixing satisfies $s_{j,z}
  =\frac{a}{a+d}\,q
  \ge\frac{a}{a+d}\,\frac{a+d}{z}
  =\frac az
  \ge\frac r4$.
Taking the maximum over the bonus grid yields
$\max_{z\in\cZ_\eps}s_{j,z}
  \ge\frac14\,\wh q_j^{\pi^{\rm}}$.
\end{proof}

\begin{lemma}[Aggregating explorers preserves exploration]
\label{lem:drop}
Let $\eps\in (0,1/8]$. Suppose that, $0\le s_j\le1$ and $q_j^{\sigma_j}\ge s_j-\frac{\eps}{16}\; \forall j$. Then the relative-drop mixture $B'$ is well defined and satisfies $q_k^{B'} \ge \max\left\{\eps,\frac{R_k}{2\log(1/\eps)}\right\}\; \forall k$.
\end{lemma}

\begin{proof}
By construction, $1=R_1\ge R_2\ge\cdots\ge R_{n+1}=\eps$.
Hence the normalizing constant satisfies
\[
  1-\eps
  \le Z
  =\sum_{i=1}^n\frac{R_i-R_{i+1}}{R_i}
  \le\sum_{i=1}^n\log\frac{R_i}{R_{i+1}}
  =\log(1/\eps).
\]
In particular, $Z>0$, so $B'$ is well defined.

If $\alpha_i>0$ and $i\ge2$, then $R_i>R_{i+1}$, and
the envelope definition gives $R_i=s_i$. Therefore,
\[
  (1-\eps)q_i^{\sigma_i}
  \ge(1-\eps)(R_i-\eps/16)
  \ge\frac{15}{16}(1-\eps)R_i
  \ge\frac{R_i}{2},
\]
where we used $R_i\ge\eps$ and $\eps\le1/8$.
The same bound holds for $i=1$, since
$q_1^{\sigma_1}=R_1=1$.

Fix a box $k$. Reaching a box $i\ge k$ entails reaching
$k$, so $q_k^{\sigma_i}\ge q_i^{\sigma_i}$. Consequently,
\[
\begin{aligned}
  q_k^{B'}
  &=\eps+\frac{1-\eps}{Z}
       \sum_{i=1}^n\alpha_iq_k^{\sigma_i}\ge\eps+\frac{1}{2Z}
       \sum_{i=k}^n\alpha_iR_i\\
  &=\eps+\frac{R_k-\eps}{2Z}\ge\frac{R_k}{2Z}
  \ge\frac{R_k}{2\log(1/\eps)}.
\end{aligned}
\]
Here the second equality follows from
$\alpha_iR_i=R_i-R_{i+1}$ and telescoping, while the
following inequality uses $Z\ge1-\eps\ge1/2$. Finally, the full-traversal component gives
$q_k^{B'}\ge\eps$, completing the proof.
\end{proof}

\section{Concentration Bound}
\label{sec:statistics}
In this section, we will provide a concentration bound at each phase. Before we describe our main lemma, we need to introduce a few notations.

\subsection{Notations}

We first introduce some definitions for the analysis. A \emph{profile} is the vector
$\mathbf m=(m_1,\ldots,m_n)$ of retained sample counts.
For a fixed deterministic profile, write $\mathcal S_i=(Y_{i,1},\ldots,Y_{i,m_i})$ and $\mathcal S=(\mathcal S_1,\ldots,\mathcal S_n)$. A \emph{procedure} is a fixed rule
that maps these sample arrays to a deterministic threshold policy. Next, we introduce \emph{triangular procedure}.

\begin{definition}[Triangular procedure]
For a fixed deterministic profile, a procedure
generating a threshold policy is \emph{triangular} if, for each
$i<n$, its threshold $\tau_i$ depends only on the sample arrays
$\mathcal S_{i+1},\ldots,\mathcal S_n$ and information fixed
before the current batch. It cannot use samples from box $i$
or earlier boxes to choose that threshold.
\end{definition}
We can see threshold policy calculated by backward induction satisfies triangular procedure. Next, we introduce \emph{witness policy}, which plays an essential in our analysis. Witness policy for box $j$, $\pi^{(j)}$, approximately maximizes the probability of
reaching box $j$ among all $4\eps$-near-optimal threshold mixtures.

\begin{definition}[Witness policy]
    For each box $j\in[n]$, a policy
$\pi^{(j)}\in\Pi$ satisfying $\Delta(\pi^{(j)})\le4\eps$ and $q_j^{\pi^{(j)}} \ge \sup_{\substack{\mu\in\Pi\\\Delta(\mu)\le4\eps}}
  q_j^\mu-\frac{\eps}{16}$ is called witness policy.
\end{definition}

Any witness policy can be chosen by the mixture of two threshold policy by the following lemma.

\begin{lemma}
\label{lem:two-component-witness}
For every $\mu\in\Pi$ and $j\in[n]$, there is a mixture $\nu$ of at most two threshold policies $\rho_1,\rho_2 \in \Pi$ satisfying $V(\nu)=V(\mu)$ and $q_j^\nu\ge q_j^\mu$.
\end{lemma}

\begin{proof}
Write $\mu=\sum_{k=1}^K w_k\rho_k$ with any $w_k>0$.
If $K\le2$, take $\nu=\mu$. Otherwise, since $K>2$, there exists a nonzero vector $c$ with $\sum_k c_k=0$ and $\sum_k c_kV(\rho_k)=0$. Choose its sign so that $\sum_k c_kq_j^{\rho_k}\ge0$. Replace each weight $w_k$ by $w_k+tc_k$, increasing $t$ from
zero until a weight first reaches zero. Such a time exists because $c$ has a negative coordinate. The new weights remain nonnegative and sum to one. By linearity, this adjustment
preserves value and does not decrease reach to box $j$. Remove the zero-weight components and repeat until at most two components remain. Their mixture is the desired $\nu$.
\end{proof}
By \cref{lem:two-component-witness}, we may choose $\pi^{(j)}=\lambda_j\rho_{j,1}+(1-\lambda_j)\rho_{j,2}$, $\rho_{j,1},\rho_{j,2}\in\Pi$ and $\lambda_j\in[0,1]$. Let $\pi_0$ denote the empirical baseline and let $\pi_{j,z}$
denote the bonus policy for a fixed target $j$ and bonus
$z\in\cZ_\eps$. Let $\mathfrak P$ be the finite family consisting of the procedures that
generate $\pi_0$ and all $\pi_{j,z}$, together with the constant procedures
that return $F$, $\pi^\star$, and the witnesses decomposition $\rho_{j,1},\rho_{j,2}$. We can see $|\mathfrak P|
\le O(n[1+\log(1/\eps)])$. Then define $\cC=\operatorname{conv}\!\left(
  \{\pi_0,F,\pi^\star\}
  \cup
  \{\pi_{j,z}:2\le j\le n,\ z\in\cZ_\eps\}
  \cup
  \{\rho_{j,1},\rho_{j,2}:j\in[n]\}
\right)$, where convex combinations are interpreted as initial mixtures. It is closed under mixing with $\pi^\star$. 

For a fixed profile and a mixture
$\mu\in\Pi$, let $S_\mu\in[n]$ denote its stopping box.
Define the analysis-only quantities $C(\mu)=\E_D\!\left[\frac{1}{m_{S_\mu}}\right]$, $\wh C(\mu)=\E_{\wh D}\!\left[\frac{1}{m_{S_\mu}}\right]$, $m_{\min}=\min_{i\in[n]}m_i$. Equivalently, with $q_{n+1}^\mu=\wh q_{n+1}^\mu=0$, $C(\mu)=\sum_{i=1}^n
\frac{q_i^\mu-q_{i+1}^\mu}{m_i}$ and $\wh C(\mu)=\sum_{i=1}^n
\frac{\wh q_i^\mu-\wh q_{i+1}^\mu}{m_i}$. These quantities average the inverse sample count at the stopping box. They are larger when a policy is more likely to stop at boxes with fewer retained observations. 

In this paragraph, we introduce stopping payoff function. A \emph{stopping-payoff function} $r:[n]\times[0,1]\to[0,1]$,
written $r_i(x)=r(i,x)$, assigns score $r_i(x)$ when a policy
stops at box $i$ with value $x$. Its expected payoff is
$J^r(\mu)=\E_D[r_{S_\mu}(X_{S_\mu})]$.
For a fixed profile, we use the value payoff
$r_i(x)=x$, the $n$ reach payoffs
$r_i^{(j)}(x)=\ind\{i\ge j\}$ for $j\in[n]$, and the
inverse-count payoff $r_i(x)=\frac{m_{\min}}{m_i}$.
All $n+2$ stopping-payoff functions take values in $[0,1]$. Their expected payoffs are
$V(\mu)$, $q_j^\mu$, and $m_{\min}C(\mu)$, respectively,
with analogous identities under the empirical distribution.

Last, we set value for parameters used in this section: $G=64A$, $\eta=\frac{\eps}{16}$, $\ell=C_\ell A^2$, $\theta=\frac{64\ell}{\eps}$, $\beta=\frac{\eps}{64}$.

\subsection{Main concentration lemma}
After introducing all the necessary notations and definitions, now, we are ready to introduce the main concentration lemma, which provides accuracy guarantee for $\wh{V}(\pi)$ and $\wh{q}^{\pi}_j$ at each phase.
\begin{theorem}[Local accuracy]
\label{prop:good}
Condition on the history before phase $h$ with accuracy parameter $\eps$, and
suppose Condition 1 holds. Set $\delta_h=\frac{\delta}{2(h+1)^2}$.  With conditional probability at least $1-\delta_h$, the phase
succeeds and
\begin{equation}
  |\wh V(\mu)-V(\mu)|\le\frac{\eps_h}{16},
  \qquad
  |\wh q_j^\mu-q_j^\mu|\le\frac{\eps_h}{16},
  \label{eq:local-accuracy}
\end{equation}
simultaneously for every $\mu\in\cC$ satisfying
$\Delta(\mu)\le8\eps_h$ and every box $j$.
\end{theorem}

In order to prove it, we need the following tool.

\begin{lemma}[Finite-family stopping-payoff concentration]
\label{lem:confidence}
Fix $\theta>0$, $\zeta\in(0,1)$, and independent arrays of
$m_i$ IID samples from each $D_i$, where $m_{\min}\ge2\theta$.
Let $\mathfrak P$ be a fixed finite family of triangular
procedures, and let $\mathfrak J$ be a fixed finite family
of stopping payoffs taking values in $[0,1]$.
Set $\ell\ge
\log\frac{2|\mathfrak P|\,|\mathfrak J|\,n}{\zeta}$ and
$\beta=\frac{\ell}{\theta}$. Then, with probability at least $1-\zeta$, simultaneously
for every $r\in\mathfrak J$ and every finite initial mixture
$\mu$ of the output policies of the procedures in $\mathfrak P$,
\[
|\widehat J^r(\mu)-J^r(\mu)|
\le \theta\widehat C(\mu)+\beta.
\]
Here $\widehat J^r(\mu)
=\mathbb E_{\widehat D}[r_{S_\mu}(X_{S_\mu})]$ and $J^r(\mu)
=\mathbb E_D[r_{S_\mu}(X_{S_\mu})]$.
\end{lemma}


A proof based on moment generating function is included in \cref{app:confidence}. Next, we will apply \cref{lem:confidence} to our setting and further simplify the upper bound to $\eta=\frac{\eps}{16}$ probability based on our choice of parameter.


\begin{proof}[Proof of \cref{prop:good}]
As proved in \cref{app:batch}, with conditional probability at least
$1-\delta_h$, the count bounds hold and \cref{lem:confidence} holds for the
realized profile, all the required stopping-payoff functions, and every mixture in $\cC$.
The argument establishes concentration for all deterministic rounded
profiles before selecting the observed profile.

Fix any $\mu\in\cC$ with $\Delta(\mu)\le8\eps$ on this event. Write
$q_i=q_i^\mu$, with $q_{n+1}=0$. The count event imply
$m_i\ge M_\eps q_i/(4G)$ and $m_{\min}\geq \frac{M\eps}{2}$. Let $k$ be the last index with $q_k>\eps$.
Then
\begin{align}
 C(\mu)
 &=\sum_{i=1}^n\frac{q_i-q_{i+1}}{m_i}\notag\\
 &\le\frac{4G}{M_\eps}
       \sum_{i=1}^{k}\frac{q_i-q_{i+1}}{q_i}
       +\frac{\eps}{m_{\min}}\notag\\
 &\le\frac{6G(1+\log(1/\eps))}{M_\eps}.
 \label{eq:true-cost}
\end{align}
For $i<k$, use
$1-q_{i+1}/q_i\le\log(q_i/q_{i+1})$ and telescope. We also use $q_{k+1}\le\eps$ and contributes at most $2/M_\eps$.

Apply \cref{lem:confidence} to $r_i(x)=m_{\min}/m_i$. Its true and
empirical expectations are $m_{\min}C(\mu)$ and
$m_{\min}\wh C(\mu)$. Thus $m_{\min}|\wh C(\mu)-C(\mu)|\le\theta\wh C(\mu)+\beta$. Since $m_{\min}\ge2\theta$,
\begin{equation}
 \wh C(\mu)\le2C(\mu)+2\beta/m_{\min},\qquad
 \theta\wh C(\mu)+\beta\le2\theta C(\mu)+2\beta.
 \label{eq:cost-transfer}
\end{equation}
Combining \cref{eq:true-cost} yields
\begin{equation}
 \theta\wh C(\mu)+\beta
 \le\frac{48\ell G(1+\log(1/\eps))}{M_\eps\eta}
       +\frac\eta2
 \le\eta.
 \label{eq:local-radius}
\end{equation}
For an executed phase, $1+\log(1/\eps)\le A$ when $C_M\ge1$.
Since $\ell=C_\ell A^2$, $G=64A$, $\eta=\eps/16$, and
$M_\eps\ge C_M A^4/\eps^2$, a sufficiently large universal $C_M$
makes the first term at most $\eta/2$. The value and reach cases of
\cref{lem:confidence} completes the proof.
\end{proof}

\section{Phase Closure and Final Regret}
\label{sec:closure}
In this section, we wrap up the one phase induction proof and use it to prove main theorem (\cref{thm:main}).

\begin{proposition}[Invariant propagation]
\label{prop:closure}
Under Condition 1 and the event of
\cref{prop:good}, the phase output satisfies for every $i\in[n]$, $\Delta(B')\le8\eps$ and $q_i^{B'}\ge\max\{\frac{q_i^{\pi}}{G},\eps\}\;\text{$\forall \pi$ with $\Delta(\pi)\le4\eps$}$.
\end{proposition}


\begin{proof}
Work on the event of \cref{prop:good}, and recall that
$a=6\eps$ and $\eta=\eps/16$.
Empirical optimality and \cref{eq:empirical-mixture-value} give $\wh V(\rho)\ge v_0-a\ge\wh V(\pi^\star)-6\eps$
for $\rho=\pi_0$ and every candidate mixture $\rho=\sigma_{j,z}$.
By \cref{lem:localization}, $\Delta(\rho)\le7\eps$, so local
accuracy applies to all these policies: for every candidate mixture, 
\begin{equation}
  \begin{aligned}
    q_j^{\sigma_{j,z}}
    &\ge\wh q_j^{\sigma_{j,z}}-\eta\\
    &=(1-w_{j,z})\wh q_j^{\pi_0}+s_{j,z}-\eta
    \ge s_{j,z}-\eta.
  \end{aligned}
  \label{eq:score-transfer}
\end{equation}
Thus \cref{lem:drop} ensures that $B'$ is well defined and $q_i^{B'}\ge
    \max\left\{\eps,\frac{R_i}{2\log(1/\eps)}\right\}
    \;\forall i$. Together with linearity
and $\Delta(F)\le1$, it yields $\Delta(B')\le(1-\eps)7\eps+\eps\leq 8\eps$.



It remains to prove coverage. The case $i=1$ is immediate.
Fix $i\ge2$ and a policy $\pi$ with $\Delta(\pi)\le4\eps$.
The witness policy $\pi^{(i)}$ and local accuracy give
\[
  v_0-\wh V(\pi^{(i)})\le6\eps,
  \qquad
  q_i^\pi\le\wh q_i^{\pi^{(i)}}+\eps.
\]
Define 
$z_i=\arg\max_{z\in\mathcal Z_\epsilon}s_{i,z}$ and 
$s_i:=s_{i,z_i}
=\max_{z\in\mathcal Z_\epsilon}s_{i,z}$. By \cref{lem:bonus}, the score-maximizing choice of $z_i$
ensures that
$\wh q_i^{\pi^{(i)}}\le4s_i$ whenever
$\wh q_i^{\pi^{(i)}}\ge24\eps$.
Consequently,
\[
  \wh q_i^{\pi^{(i)}}\le24\eps+4s_i
\]
Since \cref{lem:drop} gives
$q_i^{B'}\ge\eps$ and
$s_i\le R_i\le2\log(1/\eps)\,q_i^{B'}$, we obtain
\[
  q_i^\pi
  \le25\eps+4s_i
  \le\bigl(25+8\log(1/\eps)\bigr)q_i^{B'}
  \le Gq_i^{B'}.
\]
This completes the proof.
\end{proof}

Now we are ready to present the proof of \cref{thm:main}.

\begin{proof}[Proof of \cref{thm:main}]
    Since $\Delta(F)\le1$ and $q_i^F=1$ for every $i$, the initial
behavior $B_0=F$ satisfies Condition 1 at scale $\eps_0=1/8$.
Let $\mathcal E$ denote all phase high probability guarantees for concentration hold, we have $\Pp(\mathcal E^c) \le\sum_{h\ge0}\delta_h=\frac{\pi^2}{12}\delta<\delta$.

Work on $\mathcal E$. Recall $M_\eps=\left\lceil C_M A^4/\eps^2\right\rceil$. Each completed phase $h$ contributes at most $M_{\eps_h}\Delta(B_h) \le16\eps_h M_{\eps_h}=O(A^4/\eps_h)$ to pseudo-regret. Let $H$ be the last completed phase.
The remaining rounds are fewer than $M_{\eps_H/2}$ and use
$B_{H+1}$, whose gap is at most $8\eps_H$. Therefore, 
$$\mathcal R_T
  \le16\sum_{h=0}^{H}\eps_h M_{\eps_h}
     +8\eps_H M_{\eps_H/2}
  =O(A^4/\eps_H),$$ where we use $\sum_{h=0}^{H}\eps_h^{-1}\le2\eps_H^{-1}$.
Since the last completed batch fits within the horizon, $\frac{C_M A^4}{\eps_H^2}\le M_{\eps_H}\le T$, 
and hence $\mathcal R_T=O(A^2\sqrt T)$. 

Finally, $\Reg_T=\E[\mathcal R_T]
  \le CA^2\sqrt T+T\,\Pp(\mathcal E^c)
  \le CA^2\sqrt T+\delta T$. Take $\delta=\frac{1}{T^2}$, this completes the proof.
\end{proof}

\section{Conclusion and Future Direction}
In this work, we studied repeated prophet inequality with unknown independent distributions under prefix feedback. We gave an efficient algorithm
achieving $O(\sqrt{T}\,[1+\log(nT)]^2)$ expected regret against the optimal stopping policy that knows the distributions. This matches the $\Omega(\sqrt{T})$ lower
bound up to logarithmic factors and removes the polynomial dependence
on the number of boxes from previous work. A natural direction is to sharpen the remaining logarithmic factors, particularly to determine whether any dependence on the number of boxes is necessary in the regret bound. 

\paragraph{Acknowledgement} The author thanks his advisor, Paul Valiant, provides a counterexample for the algorithm in his previous work in NeurIPS 2025. After this, the author works closely with ChatGPT 5.5 Pro, ChatGPT 5.6 Pro and ChatGPT Astra  to explore new algorithmic idea and simplify analysis. The author takes full responsibility of the correctness for all proofs.

\clearpage
\bibliographystyle{plainnat}
\bibliography{ref}

\newpage
\appendix
\section{Proof of Finite-family Stopping-payoff Concentration}
\label{app:confidence}
In this appendix, we prove the main concentration inequality (\cref{lem:confidence}) based on standard chernoff bound analysis.

\subsection{An exponential inequality}

\begin{lemma}
\label[lemma]{lem:elementary}
Let $W,A_0$ be random variables satisfying $|W|\le A_0\le1$
almost surely. For every $|t|\le1/2$,
\begin{equation}
  \E\exp\{t(W-\E W)-t^2A_0\}\le1.
  \label{eq:elementary}
\end{equation}
The same inequality holds with conditional expectations.
\end{lemma}

\begin{proof}
The elementary inequality $\log(1+u)\ge u-u^2 \;\forall |u|\le1/2$. Since $|tW|\le1/2$ and $W^2\le A_0$, we have $tW-t^2A_0 \le tW-t^2W^2 \le \log(1+tW)$.
Consequently, $\E e^{tW-t^2A_0}\le 1+t\E W \le e^{t\E W}$. Multiplying by $e^{-t\E W}$ proves the claim.
The argument also applies with conditional expectations.
\end{proof}

\subsection{An identity with empirical reach weights}

Fix a triangular procedure and a payoff sequence $r_i(x)\in[0,1]$.
For its realized stopping functions $a_i$, define true and empirical
continuations
\begin{align}
 J_{n+1}&=\wh J_{n+1}=0,\notag\\
 J_i&=P_i[a_i r_i+(1-a_i)J_{i+1}],&
 \wh J_i&=\wh P_i[a_i r_i+(1-a_i)\wh J_{i+1}],
 \label{eq:eval-recursion}
\end{align}
where $P_i$ and $\wh P_i$ denote true and empirical expectations.
All continuations lie in $[0,1]$. Although $J_{i+1}$ is unknown to the
algorithm, it depends only on suffix samples and the fixed true distributions.

Define
\begin{equation}
 f_i(x)=a_i(x)r_i(x)+(1-a_i(x))J_{i+1},\qquad
 \xi_i=(\wh P_i-P_i)f_i.
 \label{eq:xi}
\end{equation}
Adding and subtracting $\wh P_i f_i$ in \cref{eq:eval-recursion} yields
\[
 \wh J_i-J_i=\xi_i+\wh P_i[1-a_i](\wh J_{i+1}-J_{i+1}).
\]
With $\wh q_1=1$ and $\wh q_{i+1}=\wh q_i\wh P_i[1-a_i]$, telescoping gives
\begin{equation}
 \wh J_1-J_1=\sum_{i=1}^n\wh q_i\xi_i.
 \label{eq:empirical-weight-identity}
\end{equation}

\subsection{Uniform unweighted prefix bounds}


Conditional on suffix samples, put
$W_i(x)=a_i(x)(r_i(x)-J_{i+1})$. Then $|W_i(x)|\le a_i(x)\le1$, and \((\widehat P_i-P_i)J_{i+1}=0\), thus \(\xi_i=(\widehat P_i-P_i)W_i\). Write the samples from box \(i\) as
$Y_{i,1},\ldots,Y_{i,m_i}$. Apply \cref{lem:elementary} to each sample of box $i$ with
\(W=W_i(Y_{i,s}), A_0=a_i(Y_{i,s}), t=\pm\frac{\theta}{m_i}\). This is allowed because $m_i\ge2\theta$. Then we have \(\mathbb E\!\left[
\exp\left\{
\pm\frac{\theta}{m_i}
\bigl(W_i(Y_{i,s})-P_iW_i\bigr)
-\frac{\theta^2}{m_i^2}a_i(Y_{i,s})
\right\}\middle|\,\text{suffix}
\
\right]\le1\). Write $\wh t_i=\wh P_i[a_i]$.
We apply it to all  $m_i$ samples from box $i$, by independence,
\[\mathbb E\!\left[
\exp\left\{
\pm\theta\xi_i-\theta^2\frac{\widehat t_i}{m_i}
\right\}
\,\middle|\,\text{suffix}
\right]\le1.\]
 For a fixed prefix $1,\ldots,k$, first expose all arrays after $k$, then the
blocks $k,k-1,\ldots,1$. Therefore,
\[
 \E\exp\left\{
 \pm\theta\sum_{i=1}^k\xi_i
 -\theta^2\sum_{i=1}^k\frac{\wh t_i}{m_i}
 \right\}\le1.
\]
Markov's inequality and a union bound over both signs, all procedures,
stopping-payoff function, and $k$ (for all prefixes) imply with probability at least $1-\zeta$,
\begin{equation}
 \left|\sum_{i=1}^k\xi_i\right|
 \le \theta\sum_{i=1}^k\frac{\wh t_i}{m_i}+\beta
 \quad\text{simultaneously for all these choices}.
 \label{eq:prefix-emp}
\end{equation}
\subsection{Inserting empirical reach weights}

Work on the simultaneous event in \cref{eq:prefix-emp}.
Fix a component policy $\pi$ and a stopping payoff $r$,
and suppress their dependence in the notation. Set $S_k=\sum_{i=1}^k\xi_i$, $c_k=\wh q_k-\wh q_{k+1}$, with the convention $\wh q_{n+1}=0$.
Since the empirical reach probabilities are nonincreasing
and $\wh q_1=1$, we have $c_k\ge0$ and $\sum_{k=1}^n c_k=1$.
Summation by parts gives $\sum_{i=1}^n\wh q_i\xi_i
  =\sum_{k=1}^n c_kS_k$.

Applying \cref{eq:prefix-emp} and interchanging finite sums,
we obtain
\begin{align*}
  \left|\wh J_1-J_1\right|=\left|\sum_{i=1}^n\wh q_i\xi_i\right|
  &\le\sum_{k=1}^n c_k|S_k|\le\sum_{k=1}^n c_k
       \left(\theta\sum_{i=1}^k\frac{\wh t_i}{m_i}+\beta\right)\\
  &=\theta\sum_{i=1}^n\frac{\wh t_i}{m_i}
       \sum_{k=i}^n c_k+\beta=\theta\sum_{i=1}^n\frac{\wh q_i\wh t_i}{m_i}+\beta\\
  &=\theta\sum_{i=1}^n
       \frac{\wh q_i-\wh q_{i+1}}{m_i}+\beta=\theta\wh C(\pi)+\beta.
\end{align*}
Here we used
$\sum_{k=i}^n c_k=\wh q_i$ and
$\wh q_{i+1}=\wh q_i(1-\wh t_i)$.
Together with \cref{eq:empirical-weight-identity},
this proves \cref{lem:confidence} for every component policy. The mixture of component policies will retain the same upper bound. This completes the proof of \cref{lem:confidence}.



\section{Sample counts and concentration within a phase}
\label{app:batch}

This appendix shows that each phase collects enough observations and
that the concentration lemma applies to the samples retained by the
algorithm. Together, they justify
the sampling and concentration event used in \cref{prop:good}.





\subsection{Each phase collect enough observations}

This section proves that one phase collects enough observations at every box. More specifically, we will show $m_i>N_i/2\ge Mb_i/4$ and $m_{\min}\ge M\eps/2$.

Within the phase, $B$ is frozen and executions use independent fresh
randomness. Recall $b_i=q_i^B$,
$N_i$ is the number of rounds that reach box $i$, $m_i=2^{\lfloor\log_2N_i\rfloor}$. Each $N_i$ is binomial with mean $Mb_i$. Let
$N\sim\operatorname{Bin}(M,b)$ and $\mu=Mb$, Markov's inequality gives
\[
 \Pp(N\le\mu/2)
 \le e^{\mu/4}\E[e^{-N/2}]
 \le\exp\!\left\{\mu\left(\frac14+e^{-1/2}-1\right)\right\}
 \le e^{-\mu/8}.
\]
Here $\E[e^{-N/2}]=(1-b+be^{-1/2})^M$ and $1+u\le e^u$ were used.
A union bound therefore gives
\[
 \Pp\!\left(\exists i:\ N_i<Mb_i/2\right)
 \le\sum_{i=1}^n e^{-Mb_i/8}\le\delta_h/2
\]
for sufficiently large $C_M$, because $b_i\ge2\eps$ and
$M\ge C_M A^4/\eps^2$. On the complementary event all counts are nonzero,
and rounding down to powers of two yields
\[
 m_i>N_i/2\ge Mb_i/4,\qquad m_{\min}\ge M\eps/2.
\]
The inequality $m_{\min}\geq M\eps/2\ge 2\theta$ follows for a sufficiently large
universal $C_M$. Thus the selected profile is eligible for the simultaneous concentration event. Intersecting this event and the profile-uniform concentration event loses probability at most $\delta_h$ and establishes all the sampling claims used by \cref{prop:good}.

\subsection{A finite union over rounded profiles}

Recall $M$ is the total number of fresh rounds we play $B$ at each phase. Let $L_M=1+\lfloor\log_2M\rfloor$. For each box, there are exactly $L_M$ possible values. In addtion, every rounded profile is nonincreasing. The number of profiles
is upper bounded by
\begin{equation}
 |\mathcal M_M|=\binom{n+L_M-1}{L_M-1}
 \le(n+1)^{L_M},\qquad
 \log|\mathcal M_M|=O(\log M\log(n+1)).
 \label{eq:profiles}
\end{equation}
For each profile satisfying $m_{\min}\ge2\theta$, every generated
backward-induction procedure is triangular. Include the constant
procedures for $F$, $\pi^\star$, and the $2n$ witnesses decomposition. Thus
$|\mathfrak P|=O(n[1+\log(1/\eps)])$. For this profile use the $n+2$
stopping-payoff functions consisting of value, all reach payoffs, and
$r_i(x)=m_{\min}/m_i$.

Apply \cref{lem:confidence} separately to every eligible profile with
failure budget $\zeta=\delta_h/(2|\mathcal M_M|)$. A union bound gives one event
of conditional probability at least $1-\delta_h/2$ on which concentration
holds for all those profiles, every component, and all their initial
mixtures. It suffices to choose
\[
 \ell\ge
 \log\frac{4|\mathcal M_M|\,|\mathfrak P|\,(n+2)n}{\delta_h}.
\]
For an executed phase, $M\le T$, $h=O(\log T)$, and
$\log(1/\eps)=O(\log T)$. Hence $\ell=C_\ell A^2$ dominates the required
logarithm for a sufficiently large universal $C_\ell$.

\end{document}